\newif\ifisarxiv
\isarxivtrue

\ifisarxiv
    \newcommand{\confver}[1]{}
    \newcommand{\arxivver}[1]{#1}
\else
    \newcommand{\confver}[1]{#1}
    \newcommand{\arxivver}[1]{}
\fi

\ifisarxiv

\documentclass[10pt]{article}
\usepackage[margin=1in]{geometry}
\usepackage{cite}
\newcommand{\citep}[1]{\cite{#1}}

\author{Steven Jecmen \\
Carnegie Mellon University \\
\texttt{sjecmen@cs.cmu.edu} \\
\and
Nihar B. Shah \\
Carnegie Mellon University \\
\texttt{nihars@cs.cmu.edu} \\
\and 
Fei Fang \\
Carnegie Mellon University \\
\texttt{feifang@cmu.edu} \\
\and
Vincent Conitzer \\
Duke University \\
\texttt{conitzer@cs.duke.edu} \\
}
\date{}
\date{}

\else
\documentclass{article} 
\usepackage{iclr2022_conference,times}

\author{Steven Jecmen \\
Carnegie Mellon University \\
\texttt{sjecmen@cs.cmu.edu} \\
\And
Nihar B. Shah \\
Carnegie Mellon University \\
\texttt{nihars@cs.cmu.edu} \\
\And 
Fei Fang \\
Carnegie Mellon University \\
\texttt{feifang@cmu.edu} \\
\And
Vincent Conitzer \\
Duke University \\
\texttt{conitzer@cs.duke.edu} \\
}

\iclrfinalcopy 

\fi

\usepackage{hyperref}
\usepackage{xurl}
\usepackage{enumitem}
\usepackage{booktabs}
\usepackage{comment}
\usepackage{subcaption,graphicx}
\usepackage{amsfonts,amsmath,amssymb,amsthm}
\usepackage{xspace}
\usepackage{tabularx}

\usepackage{ifthen}

\title{Tradeoffs in Preventing Manipulation \\ in Paper Bidding for Reviewer Assignment}

\newtheorem{theorem}{Theorem}

\newcommand{\algo}[1]{\textsc{#1}}
\newcommand{\bidlimit}{\algo{Bid Limit}\xspace}
\newcommand{\randomdisplay}{\algo{Random Display}\xspace}
\newcommand{\cycleprevention}{\algo{Cycle Prevention}\xspace}
\newcommand{\geographicdiversity}{\algo{Geographic Diversity}\xspace}
\newcommand{\bidmodeling}{\algo{Bid Modeling}\xspace}
\newcommand{\reviewerclustering}{\algo{Reviewer Clustering}\xspace}
\newcommand{\probabilitylimit}{\algo{Probability-Limited Randomized Assignment}\xspace}

\begin{document}

\confver{\lhead{ML Evaluation Standards Workshop at ICLR 2022}} 
\maketitle

\begin{abstract}
Many conferences rely on paper bidding as a key component of their reviewer assignment procedure. 
These bids are then taken into account when assigning reviewers to help ensure that each reviewer is assigned to suitable papers. However, despite the benefits of using bids, reliance on paper bidding can allow malicious reviewers to manipulate the paper assignment for unethical purposes (e.g., getting assigned to a friend's paper). Several different approaches to preventing this manipulation have been proposed and deployed. In this paper, we enumerate certain desirable properties that algorithms for addressing bid manipulation should satisfy. We then offer a high-level analysis of various approaches along with directions for future investigation. 
\end{abstract} 

\section{Introduction} \label{sec:intro}
In peer review in computer science, paper submissions must be assigned reviewers with the expertise required to provide a high-quality review. The standard approach to this problem involves computing a similarity score for each reviewer-paper pair representing the estimated quality of review by that reviewer for that paper, incorporating both the reviewer's expertise and preferences. These similarities are computed from various components~\citep{shah2021survey}, including text-matching with the reviewer's past work~\citep{mimno07topicbased,liu14graphpropagation,rodriguez08coauthorsip,tran17expertsuggestion,charlin13tpms}, the paper and reviewer subject areas, and reviewer-provided ``bids.'' Typically, a reviewer assignment is then found that maximizes total similarity~\citep{charlin13tpms,long2013good,goldsmith07aiconf,tang10constraied,flach2010kdd,taylor08assignment}.

One major part of the similarity computation is the paper bidding process. During paper bidding, each reviewer has the option of indicating how interested they are in reviewing each of the submitted papers by choosing a ``bid'' from a list of options (e.g., ``Not willing'', ``In a pinch'', ``Willing'', ``Eager''). Reviewers make these decisions based on the paper title, subject areas, and abstract. Paper bidding is near-universally used in practice, and tends to have a major impact on the resulting reviewer assignment. At AAAI 2021~\citep{leytonbrown2022matching}: ``{\it Reviewers were assigned papers for which they bid positively (willing or eager) 77.4\% of the time. A back-of-the-envelope calculation leads us to estimate that 79.3\% of these matches may not have happened had the reviewer not bid positively.}''

However, this reliance on paper bidding opens the door for malicious reviewers to take advantage of the paper assignment process. These malicious reviewers manipulate the assignment by providing dishonest bids in order to get assigned to a target paper. This target paper may be a friend's work which the malicious reviewer wishes to provide a positive review for, or a rival's work which the malicious reviewer wants to ``torpedo''~\citep{barroga2014safeguarding,paolucci2014mechanism,akst2010hate}. Rings of colluding reviewers have been recently uncovered at a few computer science conferences, including this instance in an ACM conference~\citep{Vijaykumar_2020,littman2021collusion}:
``{\it Another SIG community has had a collusion problem where the investigators found that a group of PC members and authors colluded to bid and push for each other’s papers violating the usual conflict-of-interest rules.}'' 
Beyond bidding, malicious reviewers can also potentially modify their subject areas or their record of past work in order to achieve a desired paper assignment. However, we focus primarily on bid manipulation in this work as the easiest and most obvious avenue through which the paper assignment can be manipulated.

Possible manipulation of the paper assignment is taken seriously by major conferences (e.g., AAAI 2021~\citep{leytonbrown2022matching} and AAAI 2022~\citep{shah2021survey}), which have used a variety of approaches to prevent this sort of malicious behavior in recent years. Several techniques are described in recent research papers~\citep{jecmen2020manipulation,wu2021making,leytonbrown2022matching,shah2021survey}, while another recent work~\citep{jecmen2022dataset} provides a dataset of malicious bids for use in future research on this issue. In this paper, we take a high-level look at several of these approaches and consider: to what extent do they satisfy properties that we would want paper assignment algorithms to satisfy? We enumerate a list of desiderata for assignment algorithms and present a preliminary evaluation of the strengths and weaknesses of various proposed approaches on these desiderata.

\section{Desiderata} \label{sec:des}
The simplest approach to handling the problem of bid manipulation is simply to not use paper bidding at all, relying solely on text-matching scores (\textit{text similarities}) and subject areas for the assignment. However, bids are near-universally used in practice and some venues even assign reviewers based only on bids. This is because there are several significant benefits to considering bids when assigning papers. 
\begin{itemize} 
    \item Bids can capture aspects of a reviewer's preferences or expertise not captured by text similarities, either because the text modeling failed to accurately represent the relationship between the submission and the reviewer's past work or because relevant factors were not represented in the reviewer's past work.
    \item Bidding allows reviewers to correct erroneous text similarities by expressing interest in papers that are truly a good match but with which the reviewer has a low text similarity.
    \item Reviewers may be more likely to provide high-quality reviews for papers that they explicitly expressed interest in reviewing during bidding. This is supported by \cite{cabanac2013capitalizing}, \confver{who}\arxivver{which} found that reviewers reported higher confidence in their reviews for papers that they bid on. 
\end{itemize}
Thus, the assignment algorithms we consider here attempt to carefully use bids in order to achieve the above benefits while remaining robust against manipulation from malicious reviewers. 

Based on these objectives, we present several desirable  and potentially conflicting properties that an ideal assignment algorithm should satisfy.
\begin{enumerate}[label=(\Alph*)]
    \item \label{des:quality} \textbf{Assignment quality:} The algorithm should produce assignments with a high level of expertise, as represented by text similarities, subject areas, and bids.
    \item \label{des:expressive} \textbf{Preference expressiveness:} The algorithm should allow reviewers to express their true preferences in a flexible manner. In particular, this means that it should produce good assignments for reviewers with idiosyncratic preferences not captured by text similarities and for reviewers with erroneous text similarities.
    \item \label{des:incentive} \textbf{Incentives to bid:} The algorithm should incentivize reviewers to provide accurate bids by assigning reviewers to papers that match their own bids to some extent.
    \item \label{des:low_prob} \textbf{Low attack success rate:} The algorithm should not allow a malicious reviewer to significantly increase their probability of assignment with a specific target paper through manipulating their bids. We call this assignment probability the ``probability of successful manipulation'' and call this manipulation of bids an ``attack.''  
    \item \label{des:attack_cost} \textbf{High attack cost:} A malicious reviewer should require extra information (e.g., other reviewers' bids/text similarities) or resources (e.g., additional colluding reviewers) in order to effectively manipulate the paper assignment.
    \item \label{des:adjust} \textbf{Adjustability:} Conference program chairs should be able to easily adjust the algorithm in order to achieve a desired tradeoff between the other desiderata.
    \item \label{des:comp} \textbf{Computational scalability:} The algorithm should be feasible to run at the large scale of modern conferences (with thousands of reviewers and papers), in terms of computational resources such as runtime and memory.
\end{enumerate}
These objectives are often contradictory and cannot all be satisfied simultaneously. We instead hope for assignment algorithms that can effectively achieve a balance between them.

\section{Algorithms}

\begin{table}[t!]
\centering
\begin{tabularx}{\textwidth}{Xll} \toprule
Algorithm           & Strengths     & Weaknesses     \\ \midrule
\bidlimit          & \ref{des:quality}, \ref{des:expressive}, \ref{des:incentive}, \ref{des:comp}       & \ref{des:low_prob}, \ref{des:attack_cost} \\
\randomdisplay      & \ref{des:expressive}, \ref{des:incentive}, \ref{des:adjust}, \ref{des:comp}           & \ref{des:quality}, \ref{des:attack_cost} \\
\cycleprevention~\citep{guo2018k}      & \ref{des:expressive}, \ref{des:incentive}      & \ref{des:low_prob}, \ref{des:adjust}, \ref{des:comp} \\ 
\geographicdiversity  & \ref{des:quality}, \ref{des:expressive}, \ref{des:incentive}, \ref{des:attack_cost}     & \ref{des:low_prob}, \ref{des:adjust} \\ 
\bidmodeling~\citep{wu2021making}           & \ref{des:quality}, \ref{des:low_prob}, \ref{des:attack_cost}          & \ref{des:expressive},  \ref{des:incentive},\ref{des:adjust} \\ 
\reviewerclustering     & \ref{des:low_prob}, \ref{des:adjust}       & \ref{des:expressive}, \ref{des:incentive},\ref{des:attack_cost}, \ref{des:comp}     \\ 
\textsc{Probability-Limited \newline Randomized Assignment}~\citep{jecmen2020manipulation}      & \ref{des:quality}, \ref{des:expressive}, \ref{des:incentive}, \ref{des:adjust}, \ref{des:comp}   & \ref{des:attack_cost}   \\
\bottomrule
\end{tabularx}
\caption{Key strengths and weaknesses of algorithms.}
\label{table:des}
\end{table}

Several different approaches have been proposed for paper assignment in the presence of malicious behavior, both in practice and in the literature. 
Although these approaches take a wide variety of forms, we view each of them as an end-to-end algorithm for the paper assignment process, encompassing the solicitation of bids and other features from reviewers and ending by outputting the final paper assignment. 
In this section, we present a brief description of some of these algorithms, along with what we see as their strengths and weaknesses on the various desiderata from Section~\ref{sec:des}. These strengths and weaknesses are summarized in Table~\ref{table:des}.

\subsection{Algorithm: \bidlimit} 

\paragraph{Description:} This simple approach requires each reviewer to enter at least some number of positive bids, and may also limit the number of negative bids that can be placed. If a reviewer does not meet these bidding criteria, the assignment algorithm may down-weight their bids or ignore them entirely when computing similarities. Intuitively, if a reviewer must bid positively on several papers (and these bids are weighted heavily when computing similarities), a malicious reviewer will have high similarity with some papers other than their target paper and may be assigned to those papers instead of their target. This idea has been used at numerous conferences, including AAAI 2021 and 2022.

\paragraph{Evaluation:} On the strong side, this approach is minimally disruptive to the standard assignment process, since honest reviewers need only make additional positive bids or remove negative bids in order to meet the requirements. Thus, the approach maintains the benefits of using bids in the standard way: it finds a high-quality assignment \ref{des:quality}, and works well for reviewers with inaccurate text similarities as they can bid positively on any papers they think are truly the best fit \ref{des:expressive}. This approach has benefits even in the absence of malicious behavior as it encourages honest reviewers to provide information \ref{des:incentive}. It also makes it more likely that each paper gets several positive bids, as \cite{shah2018design} \confver{observe}\arxivver{observes} that the standard bidding process leaves many papers with very few positive bids. The algorithm requires negligible additional computation \ref{des:comp}.

As for weaknesses, this approach is not robust against malicious behavior if malicious reviewers are behaving strategically \ref{des:low_prob}, since they can choose to bid positively only on papers with which they have very low text similarity and thus are unlikely to be assigned to. Furthermore, this attack is simple to execute \ref{des:attack_cost}. While the parameter denoting the number of required bids is easily adjustable, the connection between this parameter and the algorithm's performance on other desiderata (e.g., the probability of successful manipulation) is unclear \ref{des:adjust}.

\subsection{Algorithm: \randomdisplay} \label{sec:display}

\paragraph{Description:} Under this algorithm, each reviewer is shown a randomly-chosen subset of papers during the bidding process and can only bid on these papers. A similar procedure was used for bidding at AAAI 2020, where only a limited number of papers were shown to each reviewer. Since a malicious reviewer only has a limited probability of being able to bid on their target paper, this can lower the likelihood that they succeed at getting assigned.  If desired, a conference can provide a hard limit on the probability of successful manipulation by disallowing the assignment of any reviewer to a paper not shown to them for bidding; we refer to this as the hard-constraint variant of \randomdisplay. In other words, if half of the papers are displayed to each reviewer under the hard-constraint variant, the probability of successful manipulation would be limited at $0.5$ since the target paper is not be displayed to the malicious reviewer half of the time.

\paragraph{Evaluation:} One strength is that the subset of papers shown to each reviewer should be representative of the conference as a whole, so an honest reviewer should not have difficulty finding good matches to bid on \ref{des:expressive}. An honest reviewer also has a strong incentive to bid since bids are used in the same way as under the standard assignment algorithm \ref{des:incentive}. Under the hard-constraint variant, the program chairs can easily achieve a desired maximum probability of successful manipulation by appropriately choosing the proportion of displayed papers \ref{des:adjust}. 
The algorithm requires negligible additional computation \ref{des:comp}.

On the weak side, the optimal strategy for a malicious reviewer is simple \ref{des:attack_cost}: bid positively on the target paper if it is displayed and bid negatively on all others. Further, one can show that the hard-constraint variant of \randomdisplay is dominated by  \probabilitylimit (another algorithm described later in Section~\ref{sec:problimit}), in terms of expected similarity \ref{des:quality} when they control the probability of successful manipulation at the same level. See Appendix~\ref{apdx:random-vs-q} for the formal result. Note that the \randomdisplay and the \probabilitylimit algorithms are directly comparable because they both use the same similarity objective and provide a guarantee on the probability of successful manipulation. 

Overall, the algorithm's ability to effectively limit the probability of successful manipulation \ref{des:low_prob} is unclear. Regardless of whether the hard-constraint variant is used, sufficiently limiting the probability of successful manipulation may require imposing impractical restrictions on the bidding options for honest reviewers. Furthermore, if the hard-constraint variant is not used, then a malicious reviewer may still be able to succeed even if their target paper is not displayed for bidding. By bidding negatively on all displayed papers, they may be able to lower their similarity with enough papers so that their target paper is one of the highest-similarity papers remaining (even though it was not displayed). This issue can be solved by using the hard-constraint variant, but this comes at the cost of severely restricting the assignments for honest reviewers.

\subsection{Algorithm: \cycleprevention} 

\paragraph{Description:} In some cases, malicious reviewers who have authored a paper may collude with other reviewers who have also authored a paper at the same conference. These reviewers will attempt to get assigned to each others' papers through bidding as part of a deal to benefit each other. This algorithm~\citep{guo2018k,boehmer2021combating} attempts to prevent this collusion by restricting the assignment so that it cannot contain any $2$-cycles of reviewers: that is, if Alice is assigned to review Bob's paper, then Bob cannot be assigned to review Alice's paper. $3$-cycles and larger may also be restricted if computational resources allow. This approach has been taken by AAAI 2021~\citep{leytonbrown2022matching}.

\paragraph{Evaluation:} We first consider strengths. Note that unlike most of the other algorithms we discuss, this algorithm assumes that the malicious reviewers are part of a colluding group. As mentioned in Section~\ref{sec:intro}, there is reason to believe that collusion rings are a common form of manipulation. If so, this algorithm can provide some robustness without impacting the expressiveness of bids \ref{des:expressive} or the incentives to bid \ref{des:incentive}. 

As for weaknesses, this algorithm does not do anything to stop a malicious reviewer who is not colluding with others \ref{des:low_prob}. For example, this may be a reviewer aiming to torpedo-review a rival's paper. Furthermore, this algorithm can be circumvented by groups of reviewers who decide to collude across multiple different conferences or otherwise compensate each other outside the scope of a single conference's peer review process.  
Program chairs cannot effectively adjust the algorithm to their needs, as even increasing the size of the removed cycles is computationally difficult \ref{des:adjust}. 
This computational difficulty poses a challenge for scalability \ref{des:comp}, as finding a maximum-similarity assignment subject to cycle constraints requires solving an integer program.

The impact of this algorithm on the quality of the assignment is unclear \ref{des:quality}. With enough expert reviewers for each topic, it's possible that most honest reviewers involved in a high-similarity cycle can be replaced with a similarly-qualified reviewer; at AAAI 2021, preventing 2-cycles lowered the total assignment similarity by only 0.01\%~\cite{leytonbrown2022matching}. However, the conference in question may not have a deep enough reviewer pool and this claim may not hold even if it does. 
Additionally, the difficulty of attacking this algorithm is dependent on the type of attacker \ref{des:attack_cost}.
For colluding pairs of reviewers, the algorithm is not trivial to circumvent, since either an additional collaborator must be recruited or the submission venue of one of the papers must be changed; however, large colluding groups can easily set up cycles of higher length to avoid detection.

\subsection{Algorithm: \geographicdiversity}
 
\paragraph{Description:} Like the \cycleprevention algorithm, this approach focuses on defending against malicious reviewers who collude in groups. It specifically defends against groups of colluding reviewers that are based in a single geographic region by adding some form of geographic diversity constraint on the reviewer assignment. For example, AAAI 2021 used a constraint that no two reviewers assigned to the same paper belonged to the same region~\citep{leytonbrown2022matching}, and AAAI 2022 used a constraint that at least one assigned reviewer must be from a different region as the paper's authors. This approach is motivated by the idea that colluding groups are more likely to be from a single region, since reviewers from different areas are less likely to know each other or be able to communicate easily. 

\paragraph{Evaluation:} Large conferences include reviewers from a wide range of geographic regions, and experts in any particular topic exist in many regions. Thus, a strength is that this algorithm should not impose significant limitations on the assignments for honest reviewers. The overall assignment quality \ref{des:quality}, expressiveness of bids \ref{des:expressive}, and incentive to bid \ref{des:incentive} should all remain quite high, even if some expert reviewers are blocked from their optimal assignment. For example, the geographic diversity constraint imposed by AAAI 2021 lowered the assignment similarity by only 0.85\%~\cite{leytonbrown2022matching}. Malicious reviewers who would be stopped by this algorithm can attempt to avoid detection by recruiting colluders from a different region or by changing their location and affiliation in the conference system. However, recruiting colluders from other regions may be difficult and falsified locations can be detected by careful program chairs, making effectively circumventing this algorithm difficult \ref{des:attack_cost}.

One weakness is that, like \cycleprevention, this algorithm does not defend against a malicious reviewer who is not colluding with others or against colluding reviewers who compensate each other outside of the conference's peer review process \ref{des:low_prob}. It further does not defend well against colluding groups containing reviewers from several different regions, which could have formed because the reviewers previously met in some professional setting or because reviewers have moved institutions to a different region. The program chairs can choose the specific form of geographic constraint that is desired, but cannot easily see how effectively this will prevent collusion \ref{des:adjust} since the geographical distribution of colluding groups is unknown.  
The computational cost of the algorithm depends on the exact form of geographic constraint posed \ref{des:comp}.

\subsection{Algorithm: \bidmodeling} 

\paragraph{Description:} This algorithm, proposed in \cite{wu2021making}, uses the submitted bids from all reviewers to train a linear regression model. This model aims to predict the bid value for each reviewer-paper pair as a function of various features of that reviewer-paper pair, including the text similarity and the subject area intersection. The paper assignment is then chosen to maximize the total \textit{predicted} bid value of the assigned reviewers. The authors propose further techniques to defend against groups of colluding reviewers. 

\paragraph{Evaluation:} The primary strength of the \bidmodeling algorithm is its robustness against malicious behavior \ref{des:low_prob}: assuming that malicious reviewers cannot manipulate the reviewer-paper features, \confver{\citet[Figure 1-2]{wu2021making} demonstrate}\arxivver{\cite[Figure 1-2]{wu2021making} demonstrates} that a single malicious reviewer is unable to improve their probability of assignment to a target paper using a naive attack and has limited success with a more advanced heuristic attack. Furthermore, computing an effective attack against this model requires knowledge of the features and bids of other reviewers \ref{des:attack_cost}, which is unlikely to be available to malicious reviewers. The authors also find that the text similarity and bid values of the resulting assignment remain comparable to standard assignment methods \ref{des:quality}: \bidmodeling achieves a 16\% increase in the average text-similarity score of the assignment over a standard assignment algorithm with the NeurIPS 2014 similarity function, and a 38\% increase in the average bid value of the assignment over a standard assignment algorithm using only text similarities~\cite[Table 1]{wu2021making}. 

As for weaknesses, the algorithm pays a price for this robustness in terms of its flexibility to reviewer preferences \ref{des:expressive}, as a reviewer with incorrect text similarities may find their predicted bid values to be incorrect. The algorithm also does not allow for easy tuning by the program chairs \ref{des:adjust}, since the hyperparameters are not clearly connected to any desiderata. 
Additionally, if reviewer and paper features such as the subject areas and text similarities can be strategically manipulated, this approach may not be effective. 
Computing appropriate reviewer-paper features and fitting the model will add some additional time to the assignment algorithm at scale \ref{des:comp}, but the algorithm does run in polynomial time.  

Additionally, we conducted experiments which indicate that honest reviewers may not be sufficiently incentivized to provide bids to the algorithm \ref{des:incentive}. We sample 1000 reviewers from the dataset provided in \cite{wu2021making} and for each compute the assignments that would result if they provide their honest bids and if they provide no bids. In Figure~\ref{fig:wu}, we plot the size of the symmetric difference between the set of papers assigned to this reviewer in these two cases under \bidmodeling. We see that a majority of reviewers have identical assignments under \bidmodeling, regardless of whether or not they provide bids; the mean number of papers changed is $1.394$ and the median is $0$. For comparison, we also plot in Figure~\ref{fig:neurips} the same metric under the standard paper assignment algorithm using the NeurIPS 2016 similarity function~\citep{shah2018design}; the mean change is $2.973$ and the median is $2$.

\begin{figure*}[t!]
    \centering
    \begin{subfigure}[t]{0.45\textwidth}\includegraphics[width=1\textwidth]{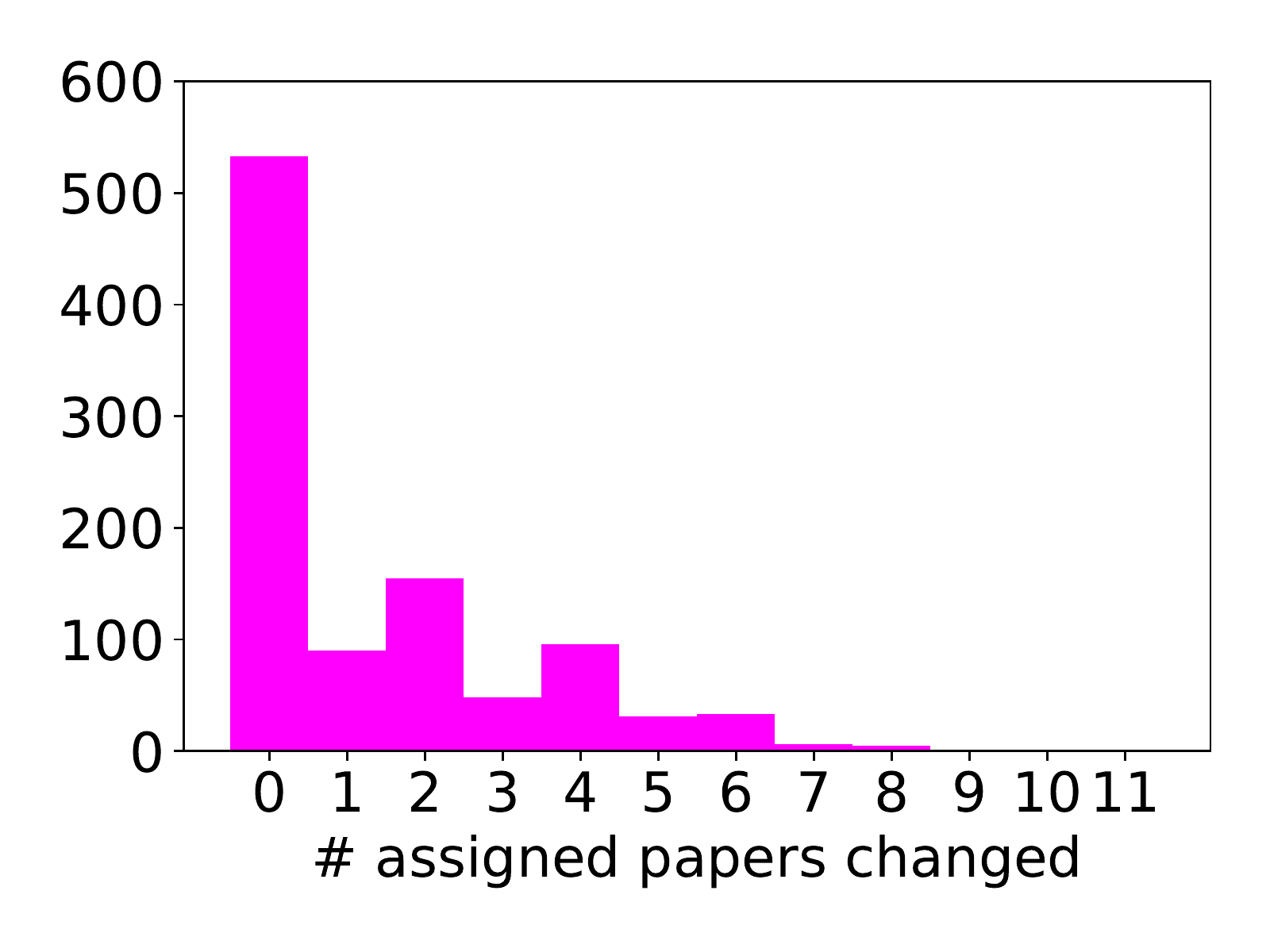}\caption{\bidmodeling.}\label{fig:wu} \end{subfigure}\quad
    \begin{subfigure}[t]{0.45\textwidth}\includegraphics[width=1\textwidth]{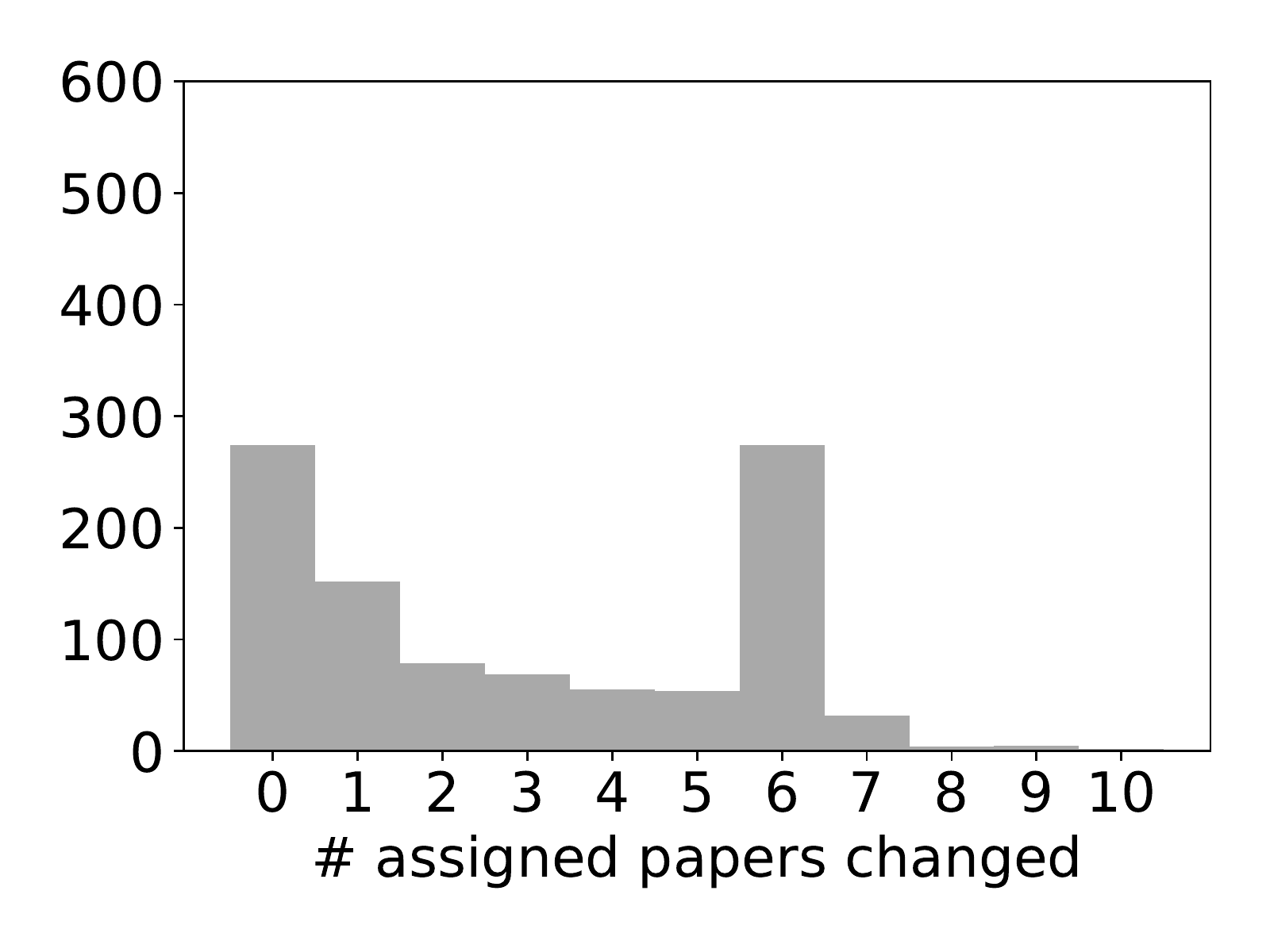} \caption{Standard assignment algorithm, NeurIPS 2016 similarity function~\citep{shah2018design}.}\label{fig:neurips} \end{subfigure} 
    \caption{Symmetric differences between the sets of papers assigned to $1000$ reviewers with honest bids and with no bids. Each reviewer is assigned at most $6$ papers and each paper is assigned to $3$ reviewers, within a dataset of around $2500$ papers and reviewers~\citep{wu2021making}. 
    } \label{fig:diffs}
\end{figure*}

\subsection{Algorithm: \reviewerclustering} 

\paragraph{Description:} Similar to \bidmodeling, this algorithm takes as input various features for each reviewer, such as their subject areas and their text similarity scores with each paper. Based on these features, it clusters reviewers into groups of some fixed size $m$. Papers are then assigned to each group based on the averaged bids of that group and randomly distributed among reviewers within the group. This algorithm is our attempt to capture some of the ideas behind \bidmodeling in a simple manner while also providing a guarantee on the maximum probability of successful manipulation: at most $1/m$. The idea of clustering reviewers by their features and arbitrarily distributing papers within each cluster is already used in contexts where reviewer assignment is done entirely by subject area~\citep{merrifield2009telescope}. 

\paragraph{Evaluation:} On the strong side, the algorithm appears to limit much of the control that a malicious reviewer has over their assignment in the same manner as \bidmodeling \ref{des:low_prob}, and it also provides a parameter that can easily be tuned to adjust the tradeoff between assignment quality and probability of successful manipulation \ref{des:adjust}. 

However, weaknesses of the algorithm are that it would not work well for reviewers with inaccurate text similarities \ref{des:expressive} and that a malicious reviewer does not require knowledge of other reviewers' features in order to determine how to bid \ref{des:attack_cost}. Further, some honest reviewers may choose to not submit bids in the hopes that the bids of their cluster are suitable enough \ref{des:incentive}. 
It could also be computationally expensive to find good fixed-size clusters, since heuristic approaches may perform poorly \ref{des:comp}.
The quality of the resulting assignment depends strongly on how well the reviewer pool can be clustered into groups of similar expertise and interests, which may vary by conference \ref{des:quality}.

\subsection{Algorithm: \probabilitylimit} \label{sec:problimit}

\paragraph{Description:} This algorithm, proposed in \cite{jecmen2020manipulation}, adds a randomized aspect to the standard assignment algorithm. Like the standard assignment algorithm, it takes bids and computes similarities as normal. Then, given a parameter $q \in [0, 1]$, it finds a randomized assignment with maximum expected similarity, subject to the constraint that the maximum probability of any reviewer-paper assignment is at most $q$. 

\paragraph{Evaluation:} We first consider strengths. By definition, \probabilitylimit finds the assignment with highest similarity among all assignments that provide a guarantee on the maximum probability of successful manipulation \ref{des:quality}. On data from ICLR 2018, \probabilitylimit achieves 90.8\% of the standard assignment algorithm's similarity with $q=0.5$~\cite[Figure 1]{jecmen2020manipulation}. Program chairs can compute this percentage for various values of $q$ before choosing one to use in deployment, allowing them to easily control the tradeoff between the assignment similarity and the maximum probability of successful manipulation \ref{des:adjust}. 
Additionally, the algorithm's guarantees on the maximum probability of successful manipulation hold without any assumptions on the malicious reviewers' capabilities, so it is still effective even if aspects like the subject areas and text similarities can be manipulated.  
Since reviewers' bids are used without modification, the expressiveness of bids is fully preserved \ref{des:expressive} and honest reviewers are still incentivized to bid \ref{des:incentive}. 
The randomized assignment can be found with the same computational resources as the standard assignment algorithm, and sampling the assignment adds little additional overhead \ref{des:comp}.

However, one weakness of the algorithm is that it's easy for a malicious reviewer to determine their best strategy \ref{des:attack_cost}: bid the maximum value on their target paper and the minimum value on all others. In this manner, malicious reviewers may easily be able to achieve this theoretical maximum probability in practice, as demonstrated in simulations by \cite{jecmen2020manipulation}. Additionally, although \probabilitylimit is optimal in terms of similarity (subject to the constraint on the probability of successful manipulation), it remains agnostic to the computation of similarities. If some similarity components (e.g., text similarity) are believed to be more trustworthy than the bids, this algorithm may not be able to control the probability of successful manipulation as efficiently as other algorithms that leverage this distinction \ref{des:low_prob}. Although one can place greater weight on trustworthy components when computing similarities, this approach may not be the optimal way to accommodate such assumptions.

In Section~\ref{sec:display}, we mention that \probabilitylimit dominates the hard-constraint variant of \randomdisplay in terms of expected similarity. However, one downside of \probabilitylimit is that reviewers may waste time bidding on papers that they will not be assigned due to the subsequent randomization. In contrast, by doing the randomization before bidding, \randomdisplay ensures that reviewers only spend time bidding on papers for which they are eligible to be assigned.

\section{Discussion}
Addressing bid manipulation in a manner that maintains the valuable properties of paper bidding is a pressing issue, given the scale and importance of modern conferences. The approaches we consider tackle the issue in a variety of ways, with different strengths and weaknesses. The least intrusive approaches (\bidlimit, \randomdisplay, \cycleprevention, and \geographicdiversity) keep the paper assignment process largely the same as under the standard assignment algorithm, which make them easier to deploy in practice. These algorithms preserve the essential benefits of bids but may not do enough to prevent manipulation effectively, as they have not been rigorously examined.  

The other algorithms can be divided into two categories based on how they use the non-bid similarity features (e.g., text similarities). \bidmodeling, along with the related \reviewerclustering algorithm, gains significant power to stop manipulation under the assumption that these features are harder for an adversary to change. If the adversary can manipulate these features (e.g., via falsifying their TPMS profile or strategically providing subject areas~\cite[Section 4.2]{shah2021survey}), these algorithms may lose some effectiveness. In contrast, \probabilitylimit entirely abstracts away the similarity computation, ignoring any differences in the cost of manipulating different features. This algorithm thus may be most appropriate for a worst-case setting where program chairs are not willing to make assumptions about the capabilities of malicious reviewers.

\cycleprevention and \geographicdiversity specifically focus on defending against colluding reviewers, but other approaches also can be extended to handle collusion. The formulation of the \bidmodeling algorithm as proposed by \confver{\citet{wu2021making}}\arxivver{\cite{wu2021making}} includes a component that effectively prevents colluding groups of a known size from manipulating the learned model. \confver{\citet[Section 5.2]{jecmen2020manipulation} provides}\arxivver{In \cite[Section 5.2]{jecmen2020manipulation}, the authors provide} an extension to their \probabilitylimit algorithm that additionally enforces that each paper be assigned diverse reviewers, essentially combining the  \probabilitylimit and \geographicdiversity approaches. 

The algorithms we consider in this work sit at different positions on the tradeoffs between our proposed desiderata, but many other positions on these tradeoffs remain unfilled. We hope that our list of desiderata can help direct the development of additional algorithms to address bid manipulation. For example, we proposed the \reviewerclustering algorithm as a simplified variant of the \bidmodeling algorithm that improves on desideratum \ref{des:adjust}. Further study on the bid manipulation problem can improve on the balance between these various desired properties. 

In addition, some past conferences have used multiple of these approaches at the same time. AAAI 2021 used both \cycleprevention and \geographicdiversity, and AAAI 2022 used forms of \bidlimit, \geographicdiversity, and \probabilitylimit. A useful direction of future work is to develop new algorithms that combine multiple previous approaches in order to simultaneously achieve their benefits. 

Finally, our analysis indirectly compares algorithms based on whether they satisfy our desiderata. One might hope to additionally conduct some form of direct comparison between algorithms, e.g., by comparing the assignment quality of each algorithm at a given probability of successful manipulation. However, there are numerous challenges in making such a comparison. Different algorithms make different assumptions about adversary capabilities and may optimize different objectives, such that both ``probability of successful manipulation'' and ``assignment quality'' may be incomparable between algorithms. Furthermore, non-malicious reviewers may behave differently under different algorithms (e.g., by providing more bids under \bidlimit than under another algorithm). Determining from past data how these reviewers might have behaved in a different environment is difficult, as seen in the literature on valuation estimation in auctions~\cite{jiang2007bidding}. We leave addressing these challenges for future work. 

\subsubsection*{Acknowledgements}
This work was supported by NSF CAREER award 1942124, NSF CAREER award 2046640, and NSF 1763734. 
We thank Hanrui Zhang for helpful discussions.

\ifisarxiv

\bibliographystyle{unsrt}
{\small
\bibliography{bibtex}}

\else 

\bibliography{bibtex}
\bibliographystyle{iclr2022_conference}

\fi

\appendix
\section*{Appendix}
\section{Comparison of \randomdisplay and \probabilitylimit} \label{apdx:random-vs-q}
We consider the hard-constraint variant of \randomdisplay, described in Section~\ref{sec:display},  which does not allow a reviewer to be assigned to papers that were not displayed to them during bidding. Define the ``display fraction'' of \randomdisplay as the proportion of papers in the subset displayed to each reviewer. In this section, we compare the hard-constraint variant of \randomdisplay with display fraction $q$ to \probabilitylimit (from Section~\ref{sec:problimit}) with probability limit $q$, in terms of expected similarity. These algorithms are directly comparable, since both limit the maximum probability of successful manipulation at $q$. 

We first introduce some notation. Call $n$ the number of reviewers and $m$ the number of papers. Define $S \in [0, 1]^{m \times n}$ as the matrix of similarities used by both algorithms, where $S_{p, r}$ is the similarity of paper $p$ with reviewer $r$. $S$ can be computed from the bids along with other features using any method, since both algorithms are agnostic to the method of similarity computation. We assume that the bids of each reviewer are the same regardless of which algorithm is used or which papers are displayed to that reviewer. 

The following result shows that \probabilitylimit outperforms \randomdisplay in terms of expected similarity. 
\begin{theorem}
For any $q \in [0, 1]$, the expected similarity of the assignment produced by \probabilitylimit with probability limit $q$ is no less than the expected similarity of the assignment produced by the hard-constraint variant of \randomdisplay with display fraction $q$.
\end{theorem}
\begin{proof}
Define the matrix $Q \in \{0, 1\}^{m \times n}$ as the random variable representing the papers displayed to each reviewer by \randomdisplay; $Q_{p,r} = 1$ if paper $p$ is displayed to reviewer $r$. Since $qm$ of the $m$ papers are chosen uniformly at random for each reviewer, $\mathbb{E}[Q_{p, r}] = q$.  Call $Q^{(1)}, \dots, Q^{(N)}$ the possible realizations of $Q$, from which $Q$ is chosen uniformly at random. For each $i \in [N]$, define $A^{(i)} \in \{0, 1\}^{m \times n}$ as the matrix representing the assignment produced by \randomdisplay if $Q^{(i)}$ was displayed; $A^{(i)}_{p,r} = 1$ if paper $p$ is assigned to reviewer $r$.

The expected similarity of the assignment produced by \randomdisplay is
\begin{align*}
    \frac{1}{N} \sum_{i=1}^N \sum_{r = 1}^n \sum_{p = 1}^m A^{(i)}_{p, r} S_{p, r}.
\end{align*}
The matrix $F = \frac{1}{N} \sum_{i=1}^N  A^{(i)}$ satisfies $F_{p, r} \leq q$ for all entries $(p, r)$, since 
\begin{align*}
    \frac{1}{N} \sum_{i=1}^N A^{(i)}_{p, r} \leq \frac{1}{N} \sum_{i=1}^N Q^{(i)}_{p, r} 
    = \mathbb{E}[Q_{p, r}] 
    = q.
\end{align*}
Consider the randomized assignment represented by $F$, where $F_{p, r}$ represents the marginal probability of assigning paper $p$ to reviewer $r$. This randomized assignment has the same expected similarity as the assignment from \randomdisplay. Further, this is a feasible randomized assignment for \probabilitylimit with probability limit $q$, meaning that \probabilitylimit will return an assignment with at least this expected similarity. 
\end{proof}

\end{document}